\documentclass[letterpaper, 10 pt, conference]{ieeeconf}  
\IEEEoverridecommandlockouts                              
\usepackage{graphics} 
\usepackage{epsfig} 
\usepackage{mathptmx} 
\usepackage{times} 
\usepackage{amsmath} 
\usepackage{amssymb}  
\usepackage{psfrag} 
\usepackage{subcaption,tikz} 
\usepackage{url}
\usepackage[english]{babel}
\usepackage{algorithm,algorithmic}
\newtheorem{theorem}{Theorem}[section]

\newcommand{\CASE}[1]{\STATE \textbf{case} #1\textbf{:} \begin{ALC@g}}
\newcommand{\ENDCASE}{\end{ALC@g}}

\newcommand{\DEFAULT}{\STATE \textbf{default:} \begin{ALC@g}}
\newcommand{\ENDDEFAULT}{\end{ALC@g}}
\newcommand{\DEFAULTLINE}[1]{\STATE \textbf{default:} }

\title{\LARGE \bf
Safe Certificate-Based Maneuvers for Teams of Quadrotors Using Differential Flatness *
}

\author{Li Wang, Aaron D. Ames, and Magnus Egerstedt$^\dagger$
\thanks{*The work by the first and third authors was sponsored by Grant No.
1544332 from the U.S. National Science Foundation, and the work
of the second author was sponsored by Grant No. 1239055 from the U.S. National Science Foundation.}
\thanks{$^\dagger$Li Wang and Magnus Egerstedt are with the School of Electrical and Computer Engineering, Aaron D. Ames is with the School of Mechanical Engineering and the School of Electrical and Computer Engineering, Georgia Institute of Technology, Atlanta, GA 30332, USA. Email: {\tt\small \{liwang, magnus, ames\}@gatech.edu} }
}

\begin{document}
\maketitle
\thispagestyle{empty}
\pagestyle{empty}

\begin{abstract}
\textit{Safety Barrier Certificates} that ensure collision-free maneuvers for teams of differential flatness-based quadrotors are presented in this paper. Synthesized with control barrier functions, the certificates are used to modify the nominal trajectory in a minimally invasive way to avoid collisions. The proposed collision avoidance strategy complements existing flight control and planning algorithms by providing trajectory modifications with provable safety guarantees. The effectiveness of this strategy is supported both by the theoretical results and experimental validation on a team of five quadrotors.
\end{abstract}

\section{INTRODUCTION} \label{sec:intro}
Due to recent advances in the design, control, and sensing technology, teams of quadrotors have become widely used in aerial robotic platforms, e.g., \cite{hehn2015real, mellinger2011minimum}. Their ability to hover and fly agilely in three dimensional space makes quadrotors effective tools for surveillance, delivery, precision agriculture, search and rescue tasks, see e.g., \cite{valenti2007mission, zhang2012application}. When teams of quadrotors are deployed to collaboratively fulfil these higher level tasks, it is crucial to make sure that they do not collide with each other. The focus of this paper is to rectify the nominal flight trajectory, which is generated with exisiting control and planning algorithms for teams of quadrotors, in a minimally invasive way to avoid collisions.

Due to the under-actuated and intrinsically unstable nature of quadrotors, it is often challenging to generate safe trajectories for arbitrary tasks. An artificial potential field approach was used in \cite{kuriki2014consensus} to avoid inter-quadrotor collisions, which can only be qualified to work for the linearized quadrotor model, i.e., near the hovering state. Additionally, real-time trajectory generation approach, utilizing the nonlinear dynamics together with time-optimal planning algorithm, was proposed in \cite{hehn2015real}. However, it is computationally expensive to accommodate collision avoidance constraints when solving optimal control problems in real time. One remedy to this problem is to exploit the differential flatness property of quadrotors, as introduced in \cite{mellinger2011minimum,zhou2014vector}, to simplify the trajectory planning process, while still leveraging the nonlinear dynamics of the quadrotors. This property has been successfully used for flight trajectory planning in cluttered environments \cite{landry2015planning}, as well as avoiding static and moving obstacles \cite{mellinger2012mixed}.

In contrast to the aforementioned methods, the goal of this paper is to modify the trajectories in a provably safe manner that is compatible with existing control and planning techniques, while exploiting the nonlinear dynamics (allowing significant deviation from hovering state and large Euler angles) of teams of quadrotors. To achieve this objective, all collision-free states of the quadrotors are encoded in a safe set. Then, \textit{Safety Barrier Certificates} are synthesized based on the differential flatness property, and a class of non-conservative control barrier functions \cite{ames2014CBF, Xu2015Robustness, nguyen2016exponential} are used to ensure the forward invariance of the safe set. Control barrier functions were used in \cite{wu2016safety, Wu2016Quad} to avoid static/moving obstacles for a single planar or 3D quadrotor. And \textit{Safety Barrier Certificates} have been applied to teams of ground mobile robots as well for collision avoidance \cite{wang2016hetero, wang2016multiobj}. As such, in this paper, the certificates are extended to more complicated multi-quadrotor systems. 
\begin{figure}[t]
\centering
  \resizebox{3in}{!}{\includegraphics{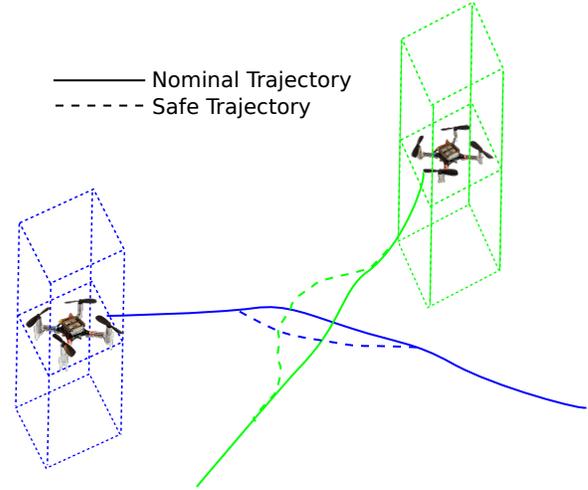}}
  \caption{Safe certificate-based flight maneuvers. The nominal trajectories generated with existing planning algorithms are modified by the safety barrier certificates to avoid collisions. The safe regions of quadrotors are modelled as rectangles to avoid both collisions and air flow disturbances.}
  \label{fig:cfwind}
\end{figure}

The main contributions of this paper are threefold: 1) \textit{Safety Barrier Certificates} are developed that provably ensure the safety of differential flatness based teams of quadrotors; 2) a strategy is developed to modify the nominal trajectory in a minimally invasive way to avoid collisions, which is compatible with existing flight control and planning algorithms; 3) the feasibility of the proposed method is demonstrated through experimental implementation of the \textit{Safety Barrier Certificates} on a team of five palm-sized quadrotors.

The rest of the paper is organized as follows: the differential flatness of quadrotor dynamics and the exponential control barrier functions are briefly revisited in Sections \ref{sec:flatquad} and \ref{sec:cbfrel}. The methods to synthesize the \textit{Safety Barrier Certificates} and minimally-invasively modify nominal trajectories are presented in Section \ref{sec:safequad}. In Section \ref{sec:feasible}, the proof of feasibility along with a virtual vehicle parameterization method to deal with actuator limits are provided. The experimental work is the topic of Section \ref{sec:exp}, and the conclusions are given in Section \ref{sec:conclude}.

\section{Differential Flatness of Quadrotor Dynamics} \label{sec:flatquad}
The quadrotor is a well-modelled dynamical system with forces and torques generated by four propellers and gravity. $Z-Y-X$ Euler angles conventions are used to define the roll ($\phi$), pitch ($\theta$), and yaw ($\psi$) angles between the quadrotor body frame and the world coordinate frame. The relevant coordinate frames and Euler angles are illustrated in Fig. \ref{fig:quadcoord}.
\begin{figure}[h] 
\centering
  \resizebox{2.5in}{!}{\includegraphics{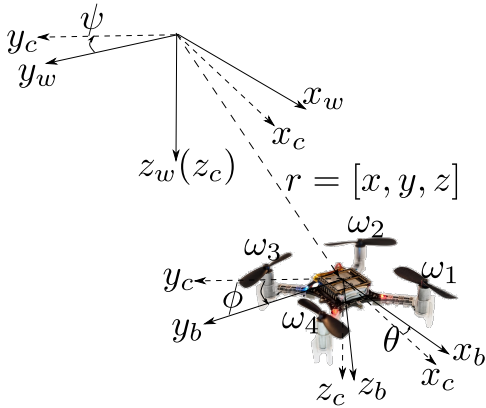}}
  \caption{Quadrotor coordinate frames. The subscripts $w$ denotes the world frame $F_w$, $b$ for the quadrotor body frame $F_b$, and $c$ for an intermediate frame $F_c$ after yaw angle rotation. $\omega_1$ to $\omega_4$ are the angular velocities of the four propellers. The palm-sized quadrotor illustrated is a Crazyflie 2.0 \cite{crazyflie} used in the experiment section.}
  \label{fig:quadcoord}
\end{figure} 

Assuming that the damping and drag-like effects are negligible \cite{hehn2015real}, the dynamics of a quadrotor is governed by the Newton-Euler equation,
\begin{eqnarray} \label{eqn:NEeqn}
m\ddot{r} = mgz_w + f_zz_b, \nonumber \\
J\dot{\omega}_b = \tau - \omega_b\times J\omega_b,\nonumber
\end{eqnarray}
where $r = [x,y,z]^T$ is the position of the center of mass in the world frame $F_w$, $\omega_b = [p,q,r]^T$ is the angular velocity in the body frame $F_b$, and $m$ and $J$ are the mass and inertia matrix of the quadrotor respectively. $f_z$ is the total thrust and $\tau = [\tau_x, \tau_y, \tau_z]$ is the torque generated by the motors. $z_w$ and $z_b$ are the unit vectors in the $Z$-direction in the world and body frames, respectively.

The quadrotor dynamics has been shown to be differentially flat in \cite{mellinger2011minimum, zhou2014vector}, i.e., the states and inputs of the system can be written in terms of algebraic functions of appropriately chosen flat outputs and their derivatives. By using the differential flatness property, trajectories can be generated by leveraging the nonlinear dynamics of the quadrotors, rather than simply viewing the system dynamics as constraints. 

As shown in \cite{zhou2014vector}, the flat output for quadrotor can be chosen as $\sigma = [x,y,z,\psi]^T$. The full state $\xi =[x,y,z,v_x,v_y,v_z,\psi,\theta,\phi,p,q,r]^T$ and input $\mu =[f_z, \tau_x, \tau_y, \tau_z]^T$ of the system can be represented algebraically using the following functions
\begin{eqnarray}\label{eqn:flatmap}
\xi &=& \beta(\sigma, \dot{\sigma}, \ddot{\sigma}, \dddot{\sigma}), \nonumber \\
\mu &=& \gamma(\sigma, \dot{\sigma}, \ddot{\sigma}, \dddot{\sigma},\ddddot{\sigma}), \nonumber
\end{eqnarray}
where we refer to \cite{zhou2014vector} for a detailed derivation and formula of the so-called endogenous transformation $(\beta,\gamma)$.

The flight trajectory can be planned in a greatly simplified flat output space with the differential flatness property. Any sufficiently smooth trajectory $\sigma(t)\in C^4$ in the flat output space can then be converted analytically back into a feasible state trajectory $\xi(t)$ and corresponding control input $\mu(t)$ using the endogenous transformation.

Switching from \cite{zhou2014vector} to this paper, in order to generate four times differentiable flight trajectory, a virtual control input $v\in\mathbb{R}^3$ is created for the integrator dynamics\footnote{Note that the trajectory generated with the integrator dynamics (\ref{eqn:fint}) is not necessarily four times differentiable. In addition, the virtual control input $v\in\mathbb{R}^3$ needs to be Lipschitz continuous for the forth derivative to exist \cite{khalil1996nonlinear}, which will be shown in Section \ref{sec:safequad}. }. For simplicity of planning, the yaw angle is always set to zero ($\psi(t)=0$),
\begin{equation}\label{eqn:fint}
\ddddot{r} = v,
\end{equation}
where $r=\sigma_{1:3}=[x,y,z]^T\in\mathbb{R}^3$. The integrator system can be equivalently written as state space form

\begin{equation}\label{eqn:fintsp}
     \dot{q}
     = \underbrace{\overbrace{\begin{bmatrix}
       0 & 1 & 0 & 0 \\[0.3em]
       0 & 0 & 1 & 0 \\[0.3em]
       0 & 0 & 0 & 1 \\[0.3em]
       0 & 0 & 0 & 0
     \end{bmatrix}}^{F\in\mathbb{R}^{4\times4}}\otimes I_{3\times3} \cdot q }_{f(q)}
       +  \underbrace{\overbrace{\begin{bmatrix}  0 \\[0.3em]  0 \\[0.3em] 0 \\[0.3em]   1 \end{bmatrix}}^{G\in\mathbb{R}^{4}}\otimes I_{3\times3}}_{g(q)} \cdot v,
\end{equation}
where $q=[r^T, \dot{r}^T, \ddot{r}^T, \dddot{r}^T]^T\in\mathbb{R}^{12}$, $\otimes$ is Kronecker product. Note that since collision avoidance requires simultaneous response of three degrees of freedom, the trajectory planning problem here can not be simplified by decoupling three independent degrees of freedom, as was done in \cite{mellinger2011minimum,hehn2015real}. 

\section{Exponential Control Barrier Functions} \label{sec:cbfrel}
With the simplified forth-order integrator model for quadrotors introduced in Section \ref{sec:flatquad}, Control Barrier Functions (CBF) can be used to ensure collision-free flight maneuvers. CBFs are Lyapunov-like functions, which can be used to provably guarantee the forward invariance of a desired set. When collision-free maneuvers are encoded as a safe set, CBFs can then be used to ensure that quadrotors never escape from the safe set, i.e., they never collide. A class of non-conservative CBFs \cite{ames2014CBF, Xu2015Robustness}, which allow the state to grow inside the safe set as opposed to strictly non-increasing as shown in Fig. \ref{fig:cbf}, is adopted to synthesize the \textit{Safety Barrier Certificates}. Consequently, quadrotor flight controllers are provided with more freedom to excise desire maneuvers while remaining safe.

Let the safe set of quadrotor states be defined as
\begin{equation}\label{eqn:setc}
\mathcal{C}_0 = \{q\in \mathbb{R}^{12}~|~ h(q) \geq 0\},
\end{equation}
where $h:\mathbb{R}^{12} \to \mathbb{R}$ is a smooth function.
\begin{figure}[h]
\centering
  \resizebox{1.8in}{!}{\includegraphics{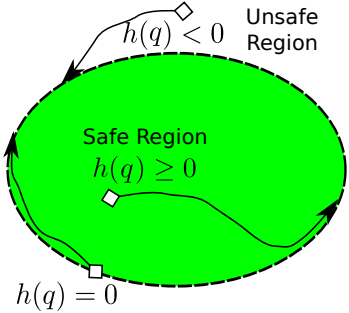}}
  \caption{Examples of non-conservative CBFs for set invariance. The diamonds and curves are example initial states and allowed state trajectories, respectively. The state is allowed to grow or even approach the boundary from inside the safe region. Outside the safe region, the state will converge asymptotically to the safe region, due to CBF constraints.}
  \label{fig:cbf}
\end{figure}

Position based safety constraints, i.e., constraints defined over $r$, are of particular interest for the quadrotor system (\ref{eqn:fintsp}), so that the quadrotor does not collide with static or moving obstacles. With a slight abuse of notation, we denote $y_0(r)=h(q)$ as the output of the system, where $h(q)$ only contains the position variable $r$. Because the virtual control input $v$ is the forth derivative of the position variable $r$, the relative degree of $y_0(r)$ is 4, which means that
\begin{equation}
y_0^{(4)}(r) = L_f^4h(q) + L_gL_f^3h(q)v,
\end{equation}
where the Lie derivative formulation stands for
\begin{equation*}
\dot{h}(q) =\frac{\partial h(q)}{\partial q}(f(q)+g(q)v) = L_fh(q)+L_gh(q)v.
\end{equation*}
Note that due to the high relative degree of $y(r)$, CBFs in \cite{ames2014CBF, Xu2015Robustness} can not be directly applied here. A variation called Exponential Control Barrier Function (ECBF) \cite{nguyen2016exponential} can however be leveraged to ensure the forward invariance of $\mathcal{C}_0$.

\textit{Definition III.1}: Given the dynamical system (\ref{eqn:fintsp}) and a set $\mathcal{C}_0$ defined in (\ref{eqn:setc}), the smooth function $h:\mathcal{C}_0\to \mathbb{R}$ with relative degree of 4 is an Exponential Control Barrier Function (ECBF) if there exists a vector $K\in \mathbb{R}^{1\times4}$ such that $\forall x\in \mathcal{C}_0,$
\begin{equation} \label{eqn:ecbf}
\sup_{u \in U} [L_f^4h(q) + L_gL_f^3h(q)v + K \eta ] \geq0,
\end{equation}
and $h(q(t))\geq C\mathrm{e}^{F-GK}\eta(q_0)\geq 0$ when $h(q_0)\geq 0$, where $\eta=[h(q), L_fh(q), L_f^2h(q), L_f^3h(q)]^T$, $C=[1,0,0,0]$.

The eligible vector $K$ can be obtained by placing the poles of the closed-loop matrix $(F-GK)$ at $p=-[p_1, p_2, ..., p_4]^T$, where $p_i>0$ for $i =1,2,3,4$. With these pole locations, a family of outputs $y_i, i =1,2,3,4$ can be defined as
\begin{equation*}
y_i =(\frac{\mathrm d}{\mathrm d t} + p_1)\circ(\frac{\mathrm d}{\mathrm d t} + p_2)\circ...\circ(\frac{\mathrm d}{\mathrm d t} + p_i)\circ h(q),
\end{equation*}
with $y_0 = h(q)$, and the associated family of super level sets
\begin{equation}
\mathcal{C}_i = \{q\in \mathbb{R}^{12} ~|~y_i(q)\geq 0\}.
\end{equation}

\begin{theorem} \label{thm:ecbf}
{\it Given a safe set $\mathcal{C}_0$ in (\ref{eqn:setc}) and associated ECBF $h(q):\mathcal{C}_0\to \mathbb{R}$, with initially $q_0 \in \mathcal{C}_i, i=0,1,2,3$ for system (\ref{eqn:fintsp}), any Lipschitz continuous controller $v(q)\in K_v(q)$ renders $\mathcal{C}_0$ forward invariant, where}
\begin{equation*}
K_v(q) = \{v\in V~|~L_f^4h(q) + L_gL_f^3h(q)v + K_v \eta \geq0\},
\end{equation*}
and $\eta=[h(q), L_fh(q), L_f^2h(q), L_f^3h(q)]^T$.
\end{theorem} 

We refer to \cite{nguyen2016exponential} for the detailed proof of general cases of this theorem. The basic idea is to design a stabilizing controller for the system using pole placement, then use Comparison Lemma \cite{khalil1996nonlinear} to show recursively that $C_i$, $i=0,1,2,3$, is forward invariant.

\section{Safety Barrier Certificates for Teams of Quadrotors}\label{sec:safequad}
Two of the main tools, i.e., differential flatness property of quadrotor and ECBF for a single quadrotor, for constructing \textit{Safety Barrier Certificates} have been revisited in sections \ref{sec:flatquad} and \ref{sec:cbfrel}. This section focuses on assembling \textit{Safety Barrier Certificates} for teams of quadrotors utilizing these tools.

\subsection{Safety Region Modelled With Super-ellipsoids} \label{sec:superel}
Consider a team of quadrotors indexed by $\mathcal{M} = \{1,2,3,\ldots, m\}$, the dynamics of quadrotors are modelled as forth-order integrators with virtual inputs $v_i\in\mathbb{R}^3$,
\begin{equation}\label{eqn:finti}
\ddddot{r}_i = v_i, ~i\in\mathcal{M}
\end{equation}
where $r_i=[x_i,y_i,z_i]^T$ is the position of the center of mass of quadrotor $i$. The full state of quadrotor $i$ is represented by $q_i=[r_i^T, \dot{r}_i^T, \ddot{r}_i^T, \dddot{r}_i^T]^T \in\mathbb{R}^{12}$. Let $r=[r_1^T,r_2^T,...,r_m^T]^T\in\mathbb{R}^{3m}$ and $v=[v_1^T,v_2^T,...,v_m^T]^T\in\mathbb{R}^{3m}$ denote the aggregate position and virtual control of the team of quadrotors.

In order to ensure the safety of the team of quadrotors, \textit{all} pairwise collisions between quadrotors need to be avoided. In addition, quadrotors can not fly directly over each other due to the air flow disturbance generated by propellers as illustrated in Fig. \ref{fig:cfwind}. During actual flights, the bottom quadrotor generally goes unstable or even crash due to the strong wind blowing from above.

To accommodate these safety requirements, each quadrotor is encapsulated with a `rectangle shape' super-ellipsoid\footnote{A super-ellipsoid is a solid geometry generally defined with the implicit function $[(\frac{x}{a})^r + (\frac{y}{b})^r]^\frac{n}{r} + (\frac{z}{c})^n\leq 1$ with $r,n\in\mathbb{R}^+$\cite{barr1981superquadrics}. $r=n=4$ is selected to approximate a `rectangle shape' here.}. Considering any pair of quadrotors $(i,j)$, the pairwise safe set is defined as
\begin{eqnarray}
\mathcal{C}_{ij} &=& \{(q_i,q_j)~|~h_{ij}(q_i,q_j)\geq 0\}, \label{eqn:setcij} \\
h_{ij}(q_i,q_j) &=& (x_i-x_j)^4 + (y_i-y_j)^4 + (\frac{z_i-z_j}{c})^4 - D_s^4,\nonumber
\end{eqnarray}
where $D_s$ is the safety distance, $c$ is the scaling factor along the Z axis caused by air flow disturbance. In practice, $c$ is obtained by flying two quadrotors over each other and identify the critical separation distance at which the bottom quadrotor goes unstable. Two quadrotors are considered safe when two `rectangle shape' super-ellipsoids do not intersect with each other as shown in Fig. \ref{fig:cf2ellip}.
\begin{figure}[h]
\centering
\begin{subfigure}{.22\textwidth}
  \centering
  \resizebox{1.84in}{!}{\includegraphics{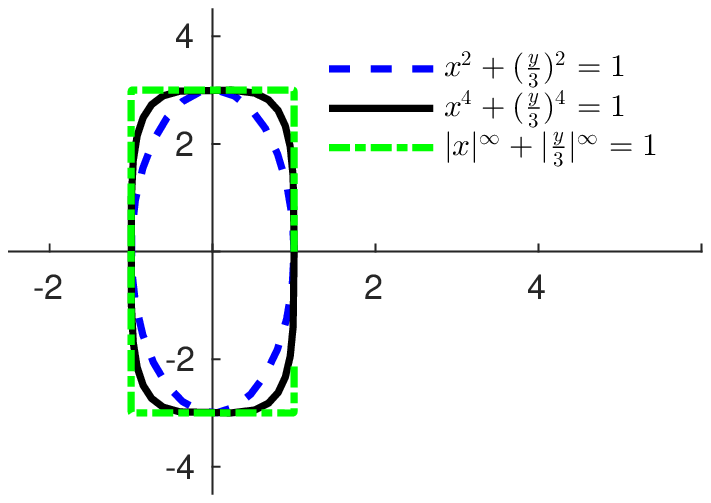}}
  \caption{XY cross section of a super-ellipsoid}
  \label{fig:f1norm}
\end{subfigure}%
\hspace{0.1in}
\begin{subfigure}{.2\textwidth}
  \centering
  \resizebox{1.0in}{!}{\includegraphics{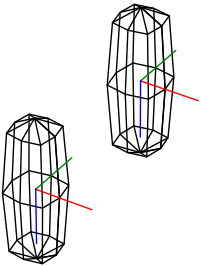}}
  \label{fig:f2frame}
  \caption{Two quadrotors modelled with super-ellipsoids}
\end{subfigure}%
\caption{Safe regions of quadrotors modelled with super-ellipsoids.}
\label{fig:cf2ellip}
\end{figure}

Note that a `rectangle shape' super-ellipsoid is chosen to appropriately approximate the safe region, since the yaw angle of quadrotors is trivially set to $\psi=0$ and quadrotors are flying with `X' configuration\footnote{`X' configuration is when the quadrotor is configured with two propellers facing forward as shown in Fig. \ref{fig:quadcoord}. `X' configuration is often favored due to improved flight agility and onboard camera view clearance. }. Alternatively, a `cylinder shape' super-ellipsoid can be picked for arbitrary yaw angles,
\begin{equation*}
\bar{h}_{ij}(q_i,q_j) = [(x_i-x_j)^2 + (y_i-y_j)^2]^\frac{n}{2} + (\frac{z_i-z_j}{c})^n - D_s^n.\nonumber
\end{equation*}

Since the pairwise safe set $\mathcal{C}_{ij}$ is defined in terms of position variables $r_{i},r_{j}$, the ECBF candidate $h_{ij}(q_i,q_j)$ has a relative degree of 4. $(q_i,q_j)$ is omitted hereafter for notation convenience. To ensure the forward invariance of $\mathcal{C}_{ij}$, virtual controls of quadrotor $i$ and $j$ need to satisfy
\begin{equation}\label{eqn:eqnvij}
\ddddot{h}_{ij} + K\cdot[h_{ij}, \dot{h}_{ij}, \ddot{h}_{ij}, \dddot{h}_{ij}]^T \geq 0,
\end{equation}
where $\ddddot{h}_{ij}$ is affine in $v_i,v_j$. Thus, the safety barrier constraint (\ref{eqn:eqnvij}) can be rearranged into a linear constraint on the virtual control when $q_i,q_j$ are given,
\begin{equation}\label{eqn:cbfconstr}
A_{ij}(q_i,q_j)\cdot v \leq b_{ij}(q_i,q_j),
\end{equation}
where $A_{ij}(q_i,q_j)=-[0,...,\underbrace{1}_{i \text{th}},...,\underbrace{-1}_{j \text{th}},...,0] \otimes [4(x_i-x_j)^3,4(y_i-y_j)^3,4\frac{(z_i-z_j)^3}{c^4}]\in\mathbb{R}^{1\times 3m}$, \\  $b_{ij}(q_i,q_j)=K\cdot[h_{ij}, \dot{h}_{ij}, \ddot{h}_{ij}, \dddot{h}_{ij}]^T + (24\dot{\delta}^4 +144\delta\circ\dot{\delta}^2\circ\ddot{\delta} + 36\bar{r}^2\circ\ddot{\delta}^2 +  48\bar{r}^2\circ\dot{\delta}\circ\dddot{\delta})\cdot \mathbf{1}_3$, $\delta=[x_i-x_j, y_i-y_j,\frac{z_i-z_j}{c}]$,\\ and $\circ$ stands for elementwise vector product.

The \textit{Safety Barrier Certificates} are formed by assembling all the pairwise safety barrier constraints
\begin{flalign}\label{eqn:ksafe}
&K_{\text{safe}} = \hfill \\
&\{v\in\mathbb{R}^{3m}~|~A_{ij}(q_i,q_j)\cdot v \leq b_{ij}(q_i,q_j), \forall i< j, i,j\in\mathcal{M}\}. \nonumber
\end{flalign}
As long as the virtual control $v$ satisfies the \textit{Safety Barrier Certificates} $K_{safe}$ and corresponding initial conditions, the team of quadrotors is guaranteed to be safe by \textit{Theorem \ref{thm:ecbf}}.

\subsection{Modifying the Nomimal Trajectory with Safety Barrier Certificates} \label{sec:reftraj}
It is often difficult to generate provably collision-free trajectories when planning the nominal trajectory for teams of quadrotors. Instead, we can first plan the flight trajectory without considering collisions, and then modify it using \textit{Safety Barrier Certificates} in a minimally invasive way to avoid collisions. Here we consider the case when a nominal trajectory $\hat{r}(t) = [\hat{r}_1^T(t), \hat{r}_2^T(t), ..., \hat{r}_m^T(t)]\in C^4$ is provided. For generality, the preplanned nominal trajectory can be generated by any methods, e.g., optimal control approach \cite{hehn2015real}, vector field approach\cite{zhou2014vector}, or parametrized curves \cite{mellinger2011minimum}, as long as it is sufficiently smooth, i.e., four times differentiable. This smooth reference trajectory $\hat{r}_i(t)$ is then tracked by a simulated integrator model using a pole placement controller with a simulated time step of $0.02s$ (simulates the 50Hz flight controller),
\begin{equation}\label{eqn:poleCLF}
\hat{v}_i = \ddddot{r}_i - K\cdot [\hat{r}_i ~\dot{\hat{r}}_i ~\ddot{\hat{r}}_i ~ \dddot{\hat{r}}_i]^T,
\end{equation}
where $K$ is picked to be the same as used for ECBFs in (\ref{eqn:eqnvij}) to trade-off tracking performance and safety enforcement.

To respect the nominal control $\hat{v}_i$ as much as possible, a quadratic program is used to minimize the difference between the actual and nominal control, which is what is meant by minimally invasive trajectory modifications,
\begin{equation}
\label{eqn:QPcontroller}
 \begin{aligned}
v^* =  & \:\: \underset{v}{\text{argmin}}
 & & J(v) = \sum_{i=1}^{N} \|{v}_{i} - \hat{v}_{i} \|^2 &\\
 & \qquad \text{s.t.}
 & & A_{ij}(q_i,q_j)v \leq b_{ij}(q_i,q_j),\: \qquad  \forall i< j,  &  \\
 &
 & &    \|  v_i\|_\infty  \leq \alpha_i,\:\qquad \forall i \in\mathcal{M},
 \end{aligned}
\end{equation}
where $\alpha_i$ is a bound on the virtual snap control. It can be observed that the actual controller $v_i$ will be the same as $\hat{v}_i$, if it is safe. The controller will only be rectified if it violates the \textit{Safety Barrier Certificates}, i.e., if it leads to collisions.

The dynamics of the simulated forth-order integrator system is integrated forward using forward Euler method. Since (\ref{eqn:poleCLF}) is a feedback tracking controller for a fully controllable integrator system, the numerical stability of the simple forward Euler method is already very reliable.  In addition, since the nominal trajectory $\hat{r}(t)$ is four times differentiable, the nominal controller $\hat{v}_i$ in (\ref{eqn:poleCLF}) will be Lipschitz continuous. According to \cite{morris2013sufficient}, the controller $v_i$ generated by the QP (\ref{eqn:QPcontroller}) will be Lipschitz continuous as well. Thus, the rectified collision-free trajectory $r(t)$ will still be four times differentiable. Differential flatness property of quadrotors can still be used to execute the rectified collision-free trajectory $r(t)$.

The advantages of using \textit{Safety Barrier Certificates} to enforce collision-free maneuvers are as follows: 
\begin{enumerate}
\item they can be used in conjunction with conventional trajectory  generation methods to provide provable safety guarantees;
\item modifications made to $\hat{r}(t)$ is minimal possible in the least-squares sense, and the rectified collision-free trajectory $r(t)$ is still four times differentiable.
\end{enumerate}

\section{Feasibility and Parameterization}\label{sec:feasible}
Section \ref{sec:safequad} provides a systematic approach to modify pre-planned trajectory in a minimally-invasive and smooth way using \textit{Safety Barrier Certificates}. However, it is not clear whether the quadratic program in (\ref{eqn:QPcontroller}) is always feasible or not. In addition, the generated trajectory might require excessive amount of control effort to execute. This section provides theoretical guarantees and parameterization method to address those issues.
\subsection{Proof of Existence of Solution}\label{sec:pffeasible}
The following theorem guarantees that a feasible safe solution to the QP problem (\ref{eqn:QPcontroller}) always exists.
\vspace{0.1in}
\begin{theorem} \label{thm:feasible}
{\it Given a team of quadrotors indexed by $\mathcal{M}$ with dynamics given in (\ref{eqn:finti}), the aggregate admissible safe control space $K_{\text{safe}}$ in (\ref{eqn:ksafe}) allowed by \textit{Safety Barrier Certificates} is guaranteed to be non-empty.}
\end{theorem} 
\begin{proof}
If any pair of quadrotors does not collide with each other, we have 
$r_i\neq r_j$ and $A_{ij}(q_i,q_j)\neq 0$, $\forall i< j$. Thus, $h_{ij}(q_i,q_j)$ are individually valid control barrier functions. However, this does not imply that a common controller exists such that all constraints are satisfied.

In order to prove that a common solution exists, let $H\in\mathbb{R}^{1\times 3m}$ be a convex combination of $-A_{ij}(q_i,q_j)\in\mathbb{R}^{1\times 3m}$,
\begin{equation}
H = -\underset{i< j}{\Sigma} \alpha_{ij} A_{ij}(q_i,q_j),
\end{equation}
where $[\alpha_{ij}]\in\mathcal{D}$, $\mathcal{D} = \{[\alpha_{ij}]~|~\underset{i< j}{\Sigma}\alpha_{ij}=1, \alpha_{ij}\geq 0\}$.

It can be observed that $H$ is the gradient of a convex function $F(r):\mathbb{R}^{3m}\to \mathbb{R}$, where $r=[r_1^T, r_2^T, ..., r_m^T]^T$ denotes the aggregate position of the quadrotors,
\begin{equation}
F(r) = \underset{i< j}{\Sigma} \alpha_{ij} [(x_i-x_j)^4+(y_i-y_j)^4+(\frac{z_i-z_j}{c})^4].
\end{equation}
Notice that $F(r)$ is non-negative and has a global minimum of 0, when $\alpha_{ij}(r_i-r_j)=0, \forall i< j$. Since all local minimums of the convex function $F(r)$ are global minimums \cite{boyd2004convex}, it can be deduced that its gradient $\nabla F(r)=H=0$ if and only if $\alpha_{ij}(r_i-r_j)=0, \forall i< j$.

Since $r_i\neq r_j$ (quadrotors do not collide), it can be further inferred that $H=0$ if and only if $\alpha_{ij}=0, \forall i< j$, which violates the fact that $\underset{i< j}{\Sigma}\alpha_{ij}=1$. Thus we have shown by contradiction that $H\neq 0$. Using this fact, it is guaranteed that $Hv + \underset{i<j}{\Sigma} \alpha_{ij} b_{ij}(q_i,q_j)>0$ always has a solution for any convex combination of $-A_{ij}(q_i,q_j)$, i.e.,
\begin{equation*}
\min_{[\alpha_{ij}]\in \mathcal{D}}\sup_{v \in \mathbb{R}^{3m}} \{\underset{i<j}{\Sigma} \alpha_{ij} [-A_{ij}(q_i,q_j) v+ b_{ij}(q_i,q_j)]\}>0.
\end{equation*}
Notice that $\mathcal{D}$ is closed and bounded, and $\mathbb{R}^{3m}$ is closed. In this case, we can exchange $\min$ and $\sup$ by using the minmax theorem \cite{rockafellar2015convex}, i.e.,
\begin{eqnarray}
&&\min_{[\alpha_{ij}]\in \mathcal{D}}\sup_{v \in \mathbb{R}^{3m}} \{\underset{i<j}{\Sigma} \alpha_{ij} [-A_{ij}(q_i,q_j) v+ b_{ij}(q_i,q_j)]\} \nonumber\\
&=&\sup_{v \in \mathbb{R}^{3m}}\min_{[\alpha_{ij}]\in \mathcal{D}} \{\underset{i<j}{\Sigma} \alpha_{ij} [-A_{ij}(q_i,q_j) v+ b_{ij}(q_i,q_j)]\} \nonumber\\
&=&\sup_{v \in \mathbb{R}^{3m}}\min_{i< j}   \{-A_{ij}(q_i,q_j) v+ b_{ij}(q_i,q_j)\}>0, \nonumber
\end{eqnarray}
which is equivalent to say that a common controller $v$ that satisfies all the pairwise barrier constraints (\ref{eqn:cbfconstr}) always exists, i.e., $K_{\text{safe}}$ is non-empty.
\end{proof}
\vspace{0.1in}
The idea of control sharing barrier function used in the proof is similar to control-sharing and merging control Lyapunov functions introduced in \cite{grammatico2014control}.

\subsection{Virtual Vehicle Parameterization} \label{sec:param}
Collision avoidance maneuvers of quadrotors might sometimes lead to significant deviations from reference trajectories. After the collision hazard disappears, excessive control effort might be required for the quadrotors to return to the reference point along the nominal trajectory. To address this issue, a virtual vehicle parameterization method proposed in \cite{egerstedt2001control} is adopted. 

The basic idea of virtual vehicle paramterization is to slow down or speed up the virtual vehicle (reference point $\hat{r}(t)$ on the nominal trajectory) as the tracking error $e_r=\|r-\hat{r}\|$ increases or decreases. In this particular application, we use the following virtual time variable to parameterize the reference point on the nominal trajectory
\begin{equation}
\dot{s} = \mathrm{e}^{-k_s\|e_r\|^2},
\end{equation}
where $k_s$ is the virtual parameterization gain. Instead of $\hat{r}(t)$, $\hat{r}(s(t))$ is fed into the \textit{Safety Barrier Certificates} rectifier shown in Fig. \ref{fig:flowchart}. Intuitively, the virtual vehicle will slow down ($\dot{s}<1$) when the tracking error is large; it will travel exactly at the desire speed ($\dot{s}=1$) when the tracking error is zero. This parameterization mechanism is intended to reduce the amount of control effort when the quadrotor has to deviate away from the virtual vehicle to avoid collisions.

To demonstrate the effectiveness of virtual vehicle parameterization, a simulation of two quadrotors flying pass each other is presented. The trajectories of two quadrotors are illustrated in Fig. \ref{fig:cf2sim}. 
\begin{figure}[h]
\centering
  \resizebox{3.2in}{!}{\includegraphics{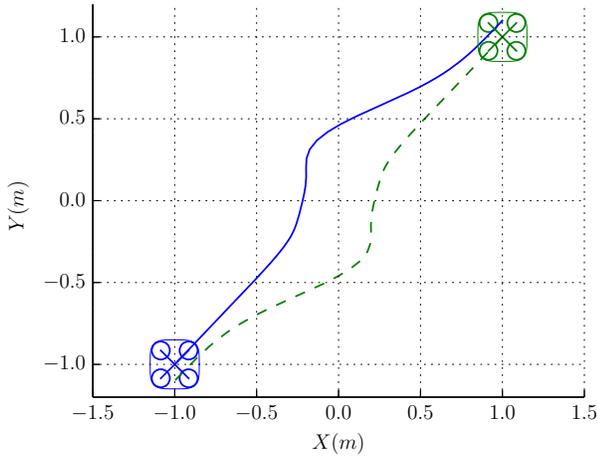}}
  \caption{Trajectories of two quadrotors flying pass each other plotted in X-Y plane. Control efforts for performing this task are illustrated in Fig. \ref{fig:cf2param}.}
  \label{fig:cf2sim}
\end{figure}

In this example, the collision avoidance manuever requires a maximum pitch angle of $70^\circ$ and a maximum thrust of $2.8$ times hovering thrust without parameterization ($k_s=0$) as shown in Fig. \ref{fig:cf2param}. In contrast, a maximum pitch angle of $25^\circ$ and a maximum thrust of $1.2$ times hovering thrust are needed with parameterization ($k_s=100$). In both cases, the desired task is accomplished within 6s.
\begin{figure*}[t!]
\centering
  \resizebox{4in}{!}{\includegraphics{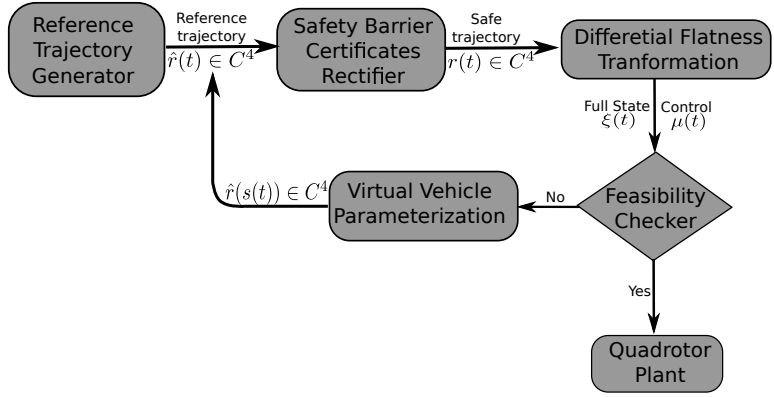}}
  \caption{Flowchart of safe trajectory generation strategy.}
  \label{fig:flowchart}
\end{figure*}
\begin{figure}[h]
\centering
  \resizebox{3.4in}{!}{\includegraphics{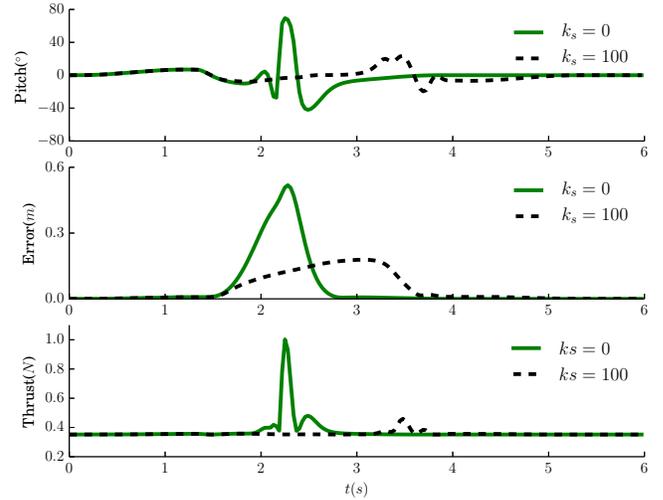}}
  \caption{Comparisons of control efforts for the quadrotor using ($k_s=100$) or without using ($k_s=0$) virtual vehicle parameterization. }
  \label{fig:cf2param}
\end{figure}

As proved in section \ref{sec:pffeasible}, the QP in (\ref{eqn:QPcontroller}) will always generate a feasible solution. However, the required control effort to avoid collisions might be excessive. In this circumstance, virtual vehicle parameterization method can be used to reduces instantaneous control efforts. By increasing the virtual parameterization gain $k_s$ significantly, the quadrotors will be granted considerably more time to perform collision avoidance manuevers. Thus, the parameterization mechanism will generate a feasible trajectory that satisfies given actuator constraints.

\subsection{Overview of Safe Trajectory Generation Strategy}
An overview of the safe trajectory generation strategy is summarized in Fig. \ref{fig:flowchart}. A smooth reference trajectory $\hat{r}(t)\in C^4$ is first fed into the safety barrier certificates rectifier, where the QP controller (\ref{eqn:QPcontroller}) is used to enforce collision-free flight maneuvers. The rectified smooth safe trajectory $r(t)\in C^4$ is then transformed into quadrotor states and controls using the differential flatness property. The full states and controls are checked to ensure that actuator limits are not violated. Otherwise, the reference trajectory is parameterized $\hat{r}(s(t))\in C^4$  and fed into the safety barrier certificates rectifier again. This process can be repeated until the virtual vehicle parameterization strategy yields appropriate flight trajectory that respects both safety and actuator constraints. In the end, the generated feasible safe trajectory is sent to execute on the team of quadrotors.

\section{Experiment} \label{sec:exp}
The \textit{Safety Barrier Certificates} are implemented on a team of five palm-sized quadrotors (Crazyflie 2.0). All communication channels between different devices and control programs are coordinated by a ROS server. The real-time positions and Euler angles of quadrotors are tracked by the Optitrack motion capture system with an update rate of 50Hz. The 50Hz quadrotor motion controller is developed based on the ROS driver for Crazyflie 2.0 built by ACTLab at USC \cite{HoenigMixedReality2015}. To ensure stable trajectory tracking behavior, Euler angles and Euler angle rates generated with the differential flatness property are sent to quadrotors as control commands. The overall quadrotor control diagram is shown in Fig. \ref{fig:ros}.
\begin{figure}[h]
\centering
  \resizebox{3.3in}{!}{\includegraphics{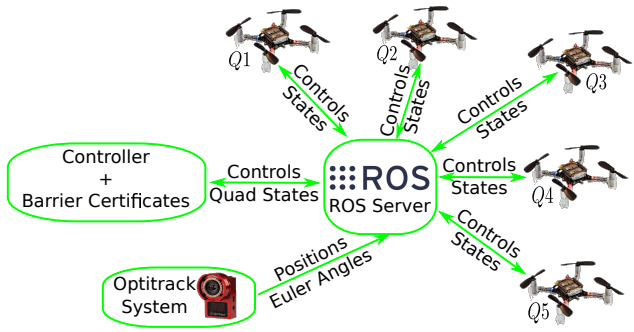}}
  \caption{Quadrotor control system diagram}
  \label{fig:ros}
\end{figure}

\subsection{Showcase: Flying Through a Static Formation}
In the first experiment, one of the quadrotor ($Q5$) is commanded to fly through a static formation consisting of four quadrotors ($Q1-Q4$) as shown in Fig. \ref{fig:exp1static}.
\begin{figure}[h]
\centering
  \resizebox{3.0in}{!}{\includegraphics{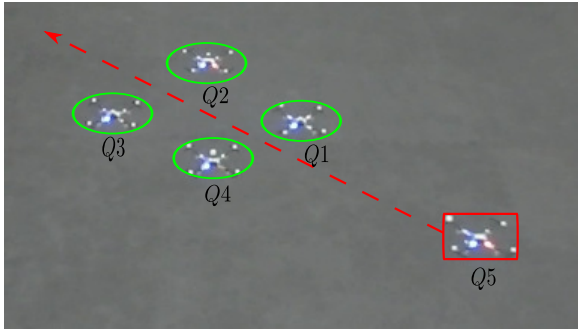}}
  \caption{Snapshot from a experiment of quad $Q5$ flying through a static formation consisting of four quads $Q1-Q4$. The video of this experiment is available online \cite{Quad:video}.}
  \label{fig:exp1static}
\end{figure}

Quadrotors ($Q1-Q4$) are designed to hover at four places with reference trajectories given as
\begin{eqnarray}
\hat{r}_1(t) =& \begin{bmatrix} 0.25 \\ 0 \\ -0.8 \end{bmatrix}, 
\hspace{0.08in} \hat{r}_2(t) =& \begin{bmatrix} 0 \\ 0.25 \\ -0.8 \end{bmatrix}, \nonumber \\
\hat{r}_3(t) =& \begin{bmatrix} -0.25 \\ 0 \\ -0.8 \end{bmatrix}, 
\hat{r}_4(t) =& \begin{bmatrix} 0 \\ -0.25 \\ -0.8 \end{bmatrix}. \nonumber
\end{eqnarray}

Another quadrotor ($Q5$) is designed to go from $p_0 = [0.6, -0.6, -0.8]^T$ to $p_1 = [-0.6, 0.6, -0.8]^T$. The nominal trajectory can be generated as
\begin{equation*}
\hat{r}_5(t) = \text{BezierInterp}(p_0,p_1),
\end{equation*}
where the \text{BezierInterp} function stands for the Bezier curve interpolation algorithm between two waypoints.

The safety distance between quadrotors is specified as $D_s=25cm$ to account for the tracking frames and controller tracking errors. Intuitively, quadrotors will collide with each other if nominal trajectories are directly executed. During the experiment, the \textit{Safety Barrier Certificates} are applied to modify the nominal trajectories in a minimally invasive way to avoid collisions. As demonstrated in Fig. \ref{fig:exp1}, $Q5$ successfully navigated through the static formation of four quadrotors within 3.4s.
\begin{figure}[H]
\centering
\begin{subfigure}{.24\textwidth}
  \centering
  \resizebox{1.7in}{!}{\includegraphics{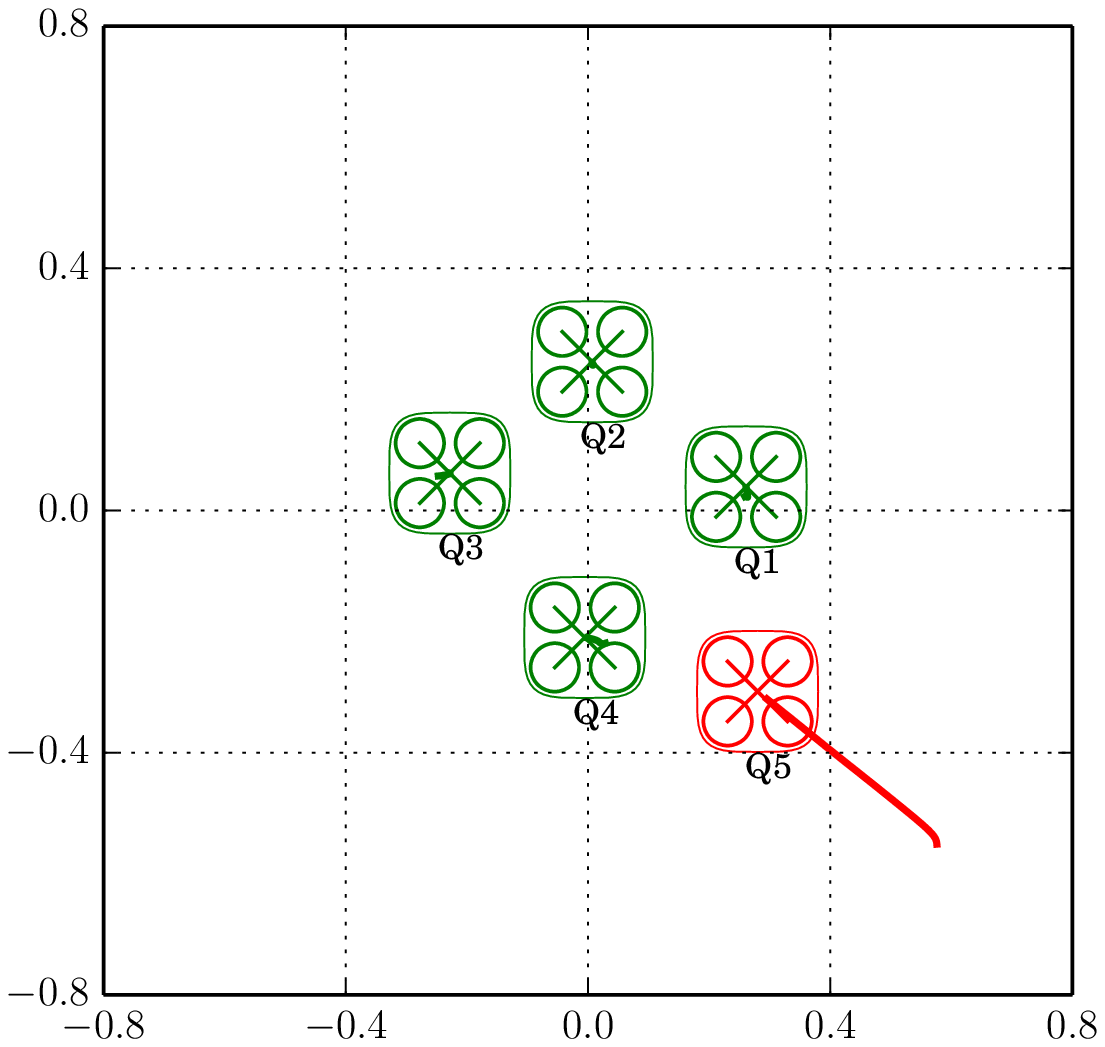}}
  \caption{Agents at 8.8s}
  \label{fig:t1s}
\end{subfigure}%
\begin{subfigure}{.24\textwidth}
  \centering
  \resizebox{1.7in}{!}{\includegraphics{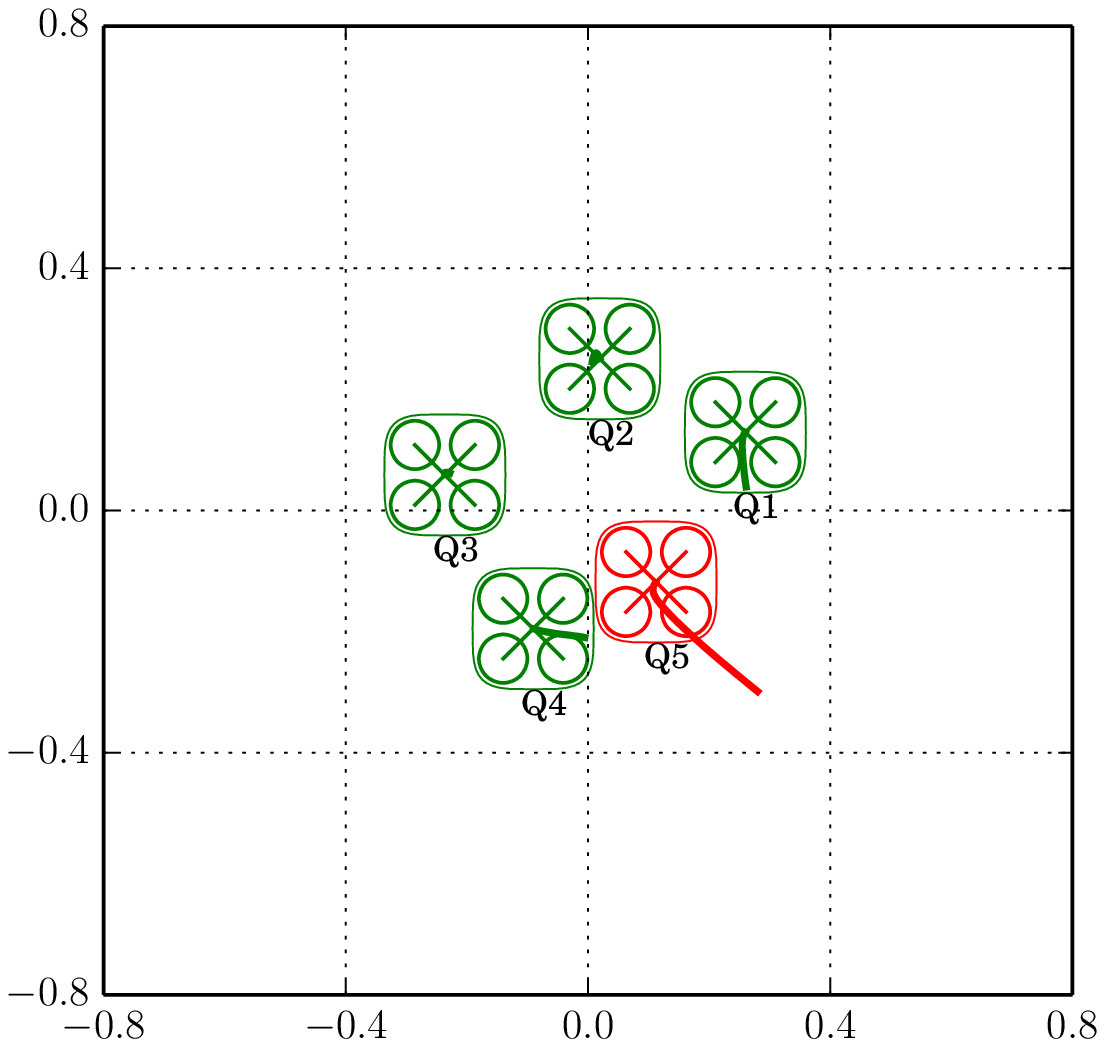}}
  \caption{Agents at 9.6s}
  \label{fig:t2s}
\end{subfigure}
\\
\begin{subfigure}{.24\textwidth}
  \centering
  \resizebox{1.7in}{!}{\includegraphics{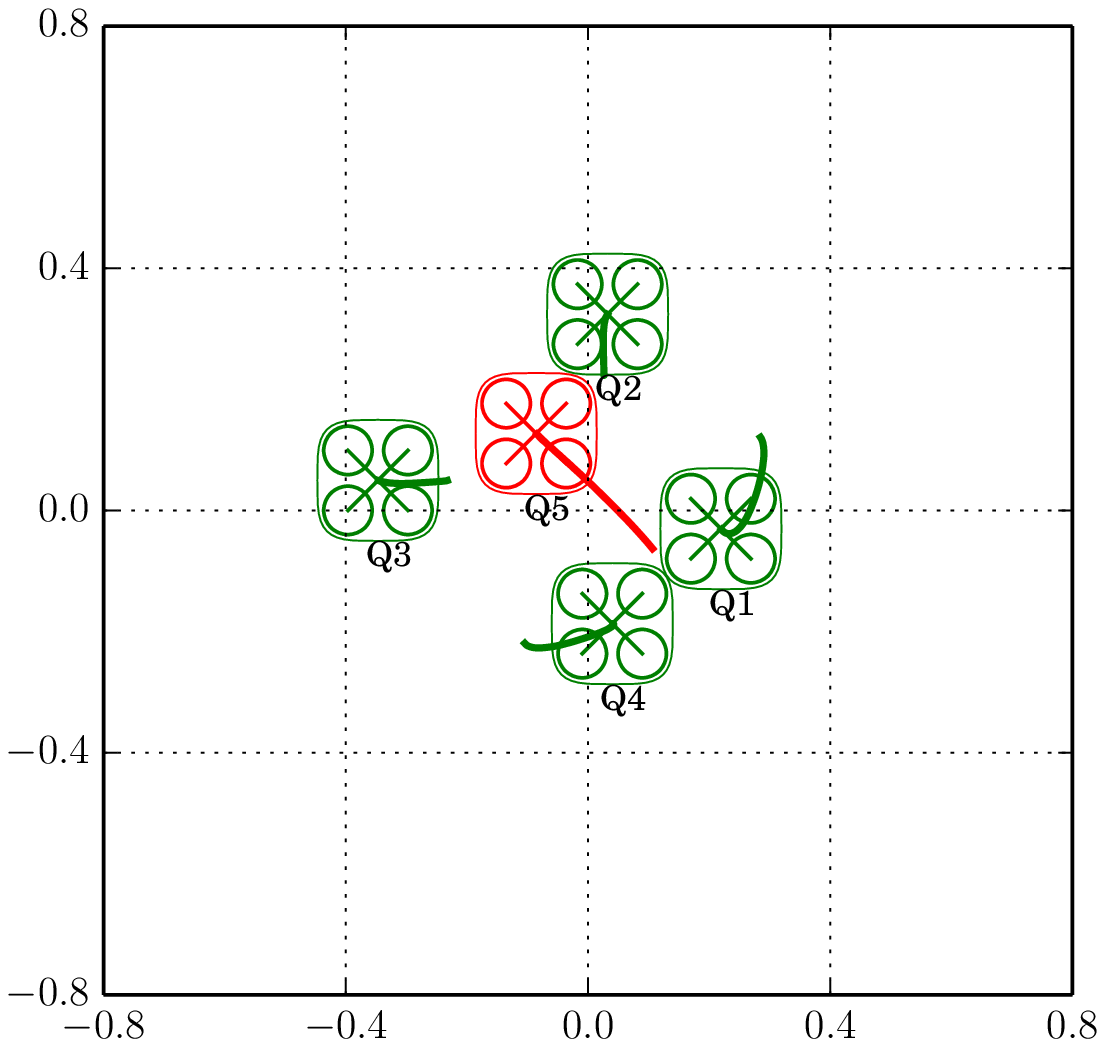}}
  \caption{Agents at 11.0s}
  \label{fig:t3s}
\end{subfigure}%
\begin{subfigure}{.24\textwidth}
  \centering
  \resizebox{1.7in}{!}{\includegraphics{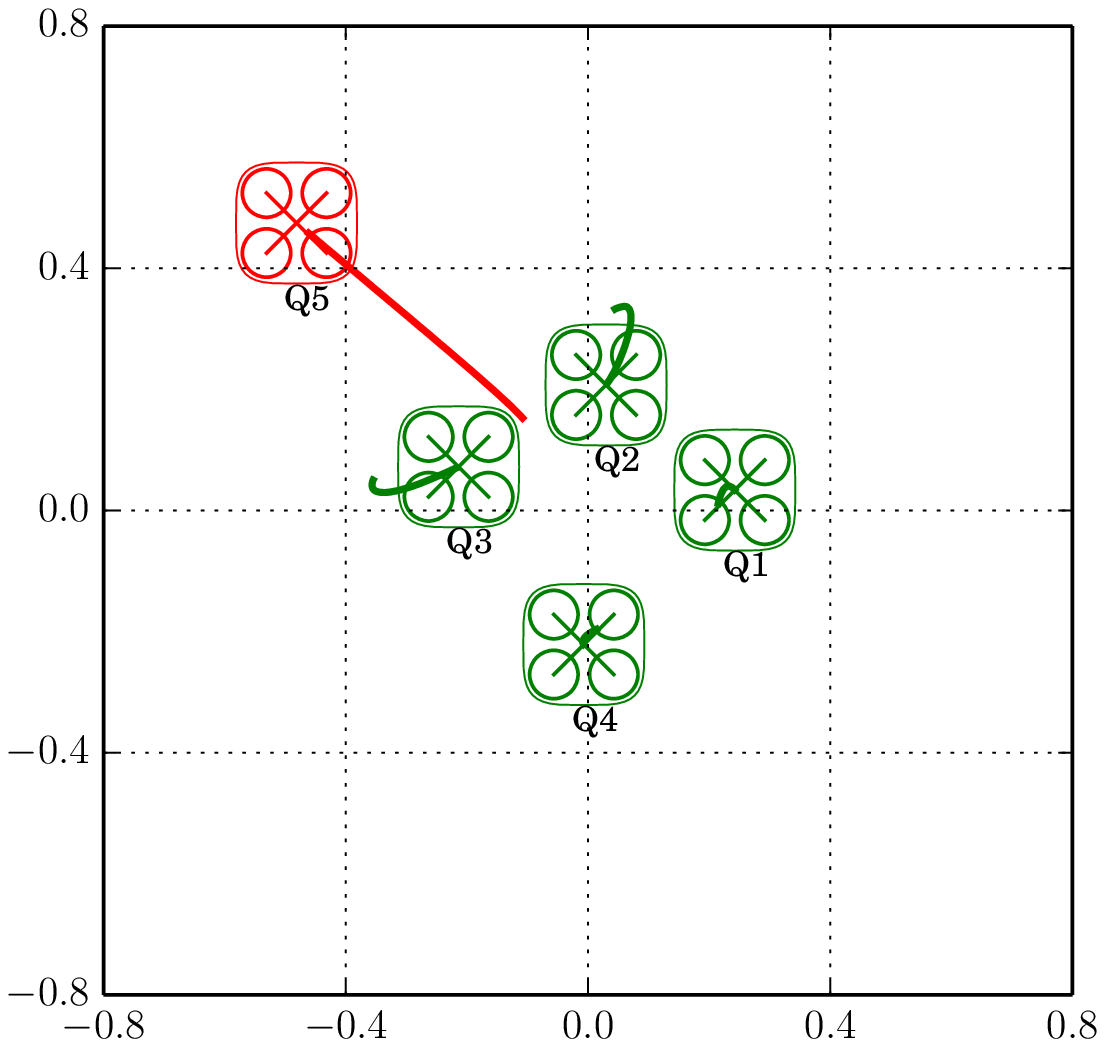}}
  \caption{Agents at 12.2s}
  \label{fig:t4s}
\end{subfigure}
\caption{Experimental data of the team of quadrotors plotted in the X-Y plane. The tail of each quadrotor illustrates its trajectory in the past 0.8s.}
\label{fig:exp1}
\end{figure}

\subsection{Showcase: Flying Through a Spinning Formation}
Similar to the previous experiment, the quadrotor $Q5$ is commanded to fly through a spinning formation as illustrated in Fig. \ref{fig:exp2spin}. 
\begin{figure}[h]
\centering
  \resizebox{3.0in}{!}{\includegraphics{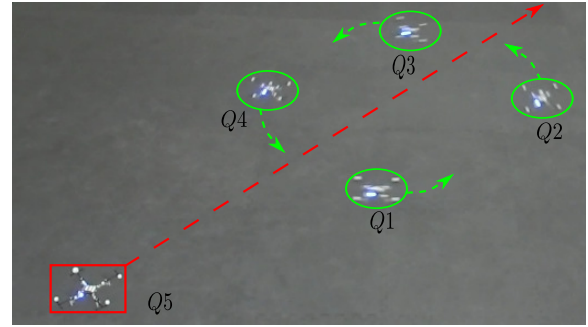}}
  \caption{Snapshot from a experiment of quad $Q5$ flying through a spinning formation consisting of four quads $Q1-Q4$. The video of this experiment is available online \cite{Quad:video}.}
  \label{fig:exp2spin}
\end{figure}

The reference trajectories of quadrotors are designed as
\begin{eqnarray}
\hat{r}_1(t) =& \begin{bmatrix} 0.45\sin(\frac{\pi}{2} t-\frac{\pi}{2}) \\ 0.45\cos(\frac{\pi}{2} t-\frac{\pi}{2}) \\ -0.8 \end{bmatrix}, 
\hat{r}_2(t) =& \begin{bmatrix} 0.45\cos(\frac{\pi}{2} t) \\ 0.45\sin(\frac{\pi}{2} t) \\ -0.8 \end{bmatrix}, \nonumber \\
\hat{r}_3(t) =& \begin{bmatrix} 0.45\cos(\frac{\pi}{2} t+\frac{\pi}{2}) \\ 0.45\sin(\frac{\pi}{2} t+\frac{\pi}{2}) \\ -0.8 \end{bmatrix}, 
\hat{r}_4(t) =& \begin{bmatrix} 0.45\cos(\frac{\pi}{2} t+\pi) \\ 0.45\sin(\frac{\pi}{2} t+\pi) \\ -0.8 \end{bmatrix}, \nonumber
\end{eqnarray}
\begin{equation*}
\hat{r}_5(t) = \text{BezierInterp}(\begin{bmatrix} -0.9 \\ -0.9 \\ -0.8 \end{bmatrix},\begin{bmatrix} 0.9 \\ 0.9 \\ -0.8 \end{bmatrix}).
\end{equation*}

As shown in Fig. \ref{fig:exp2}, $Q5$ successfully navigated through the spinning formation with minimal impact on the other four quadrotors by applying the \textit{Safety Barrier Certificates}.
\begin{figure}[H]
\centering
\begin{subfigure}{.24\textwidth}
  \centering
  \resizebox{1.7in}{!}{\includegraphics{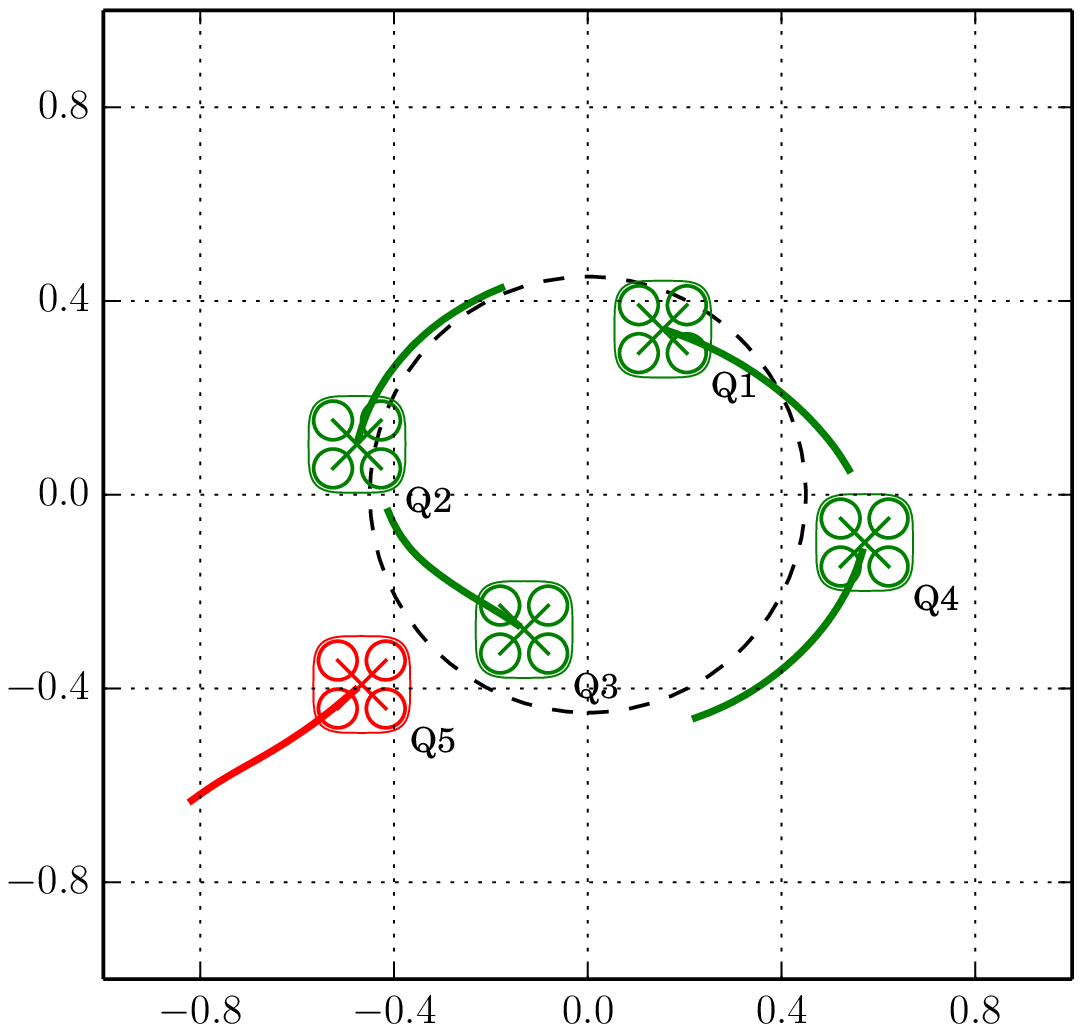}}
  \caption{Agents at 24.3s}
  \label{fig:t1}
\end{subfigure}%
\begin{subfigure}{.24\textwidth}
  \centering
  \resizebox{1.7in}{!}{\includegraphics{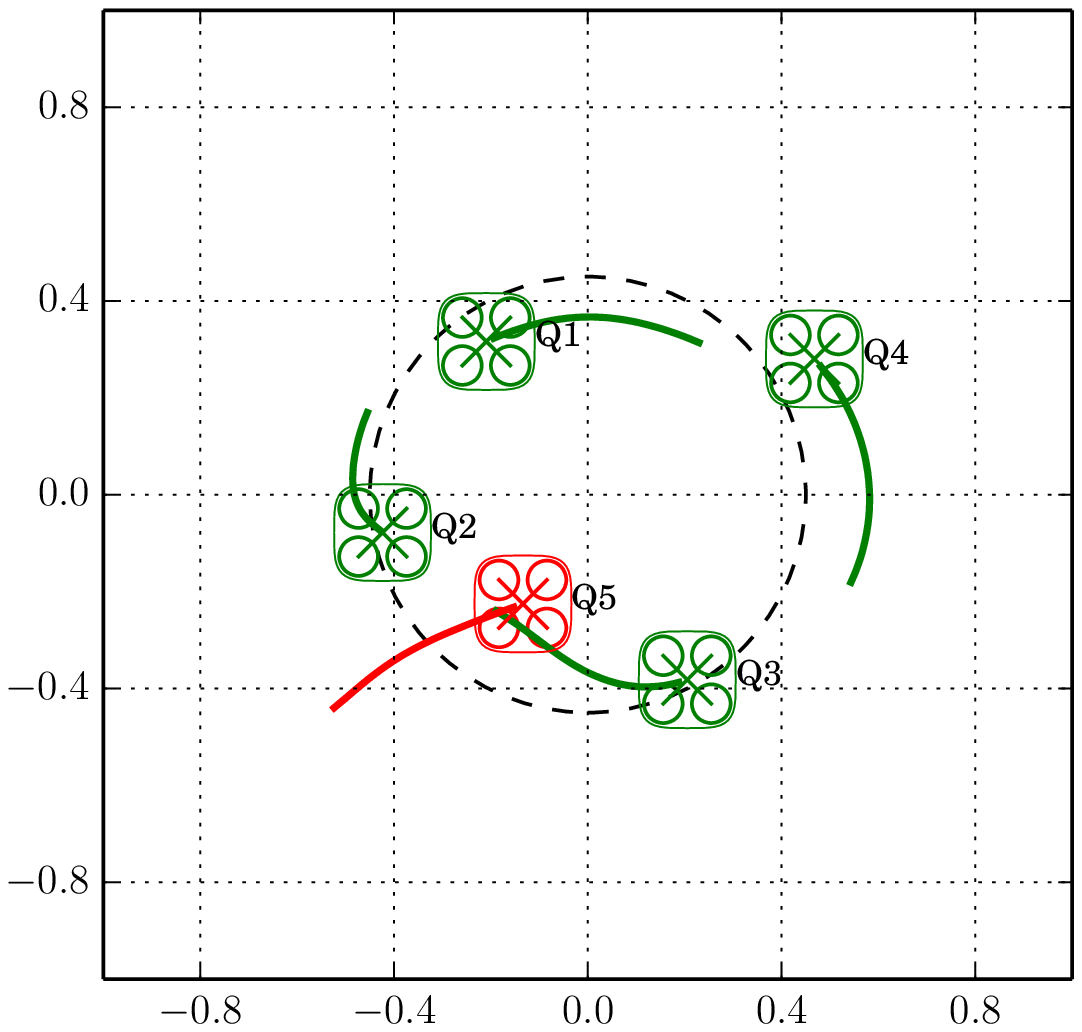}}
  \caption{Agents at 25.3s}
  \label{fig:t2}
\end{subfigure}
\\
\begin{subfigure}{.24\textwidth}
  \centering
  \resizebox{1.7in}{!}{s\includegraphics{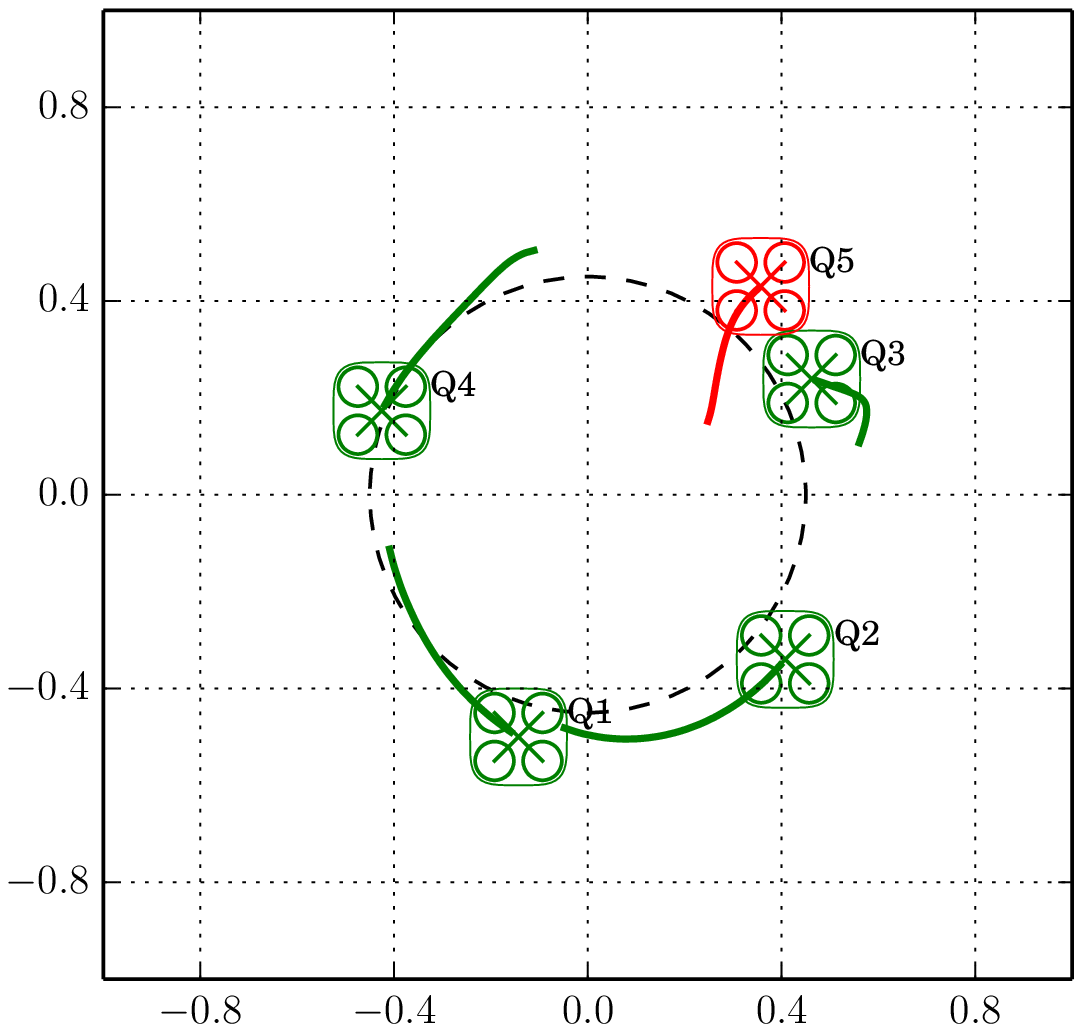}}
  \caption{Agents at 26.7s}
  \label{fig:t3}
\end{subfigure}%
\begin{subfigure}{.24\textwidth}
  \centering
  \resizebox{1.7in}{!}{\includegraphics{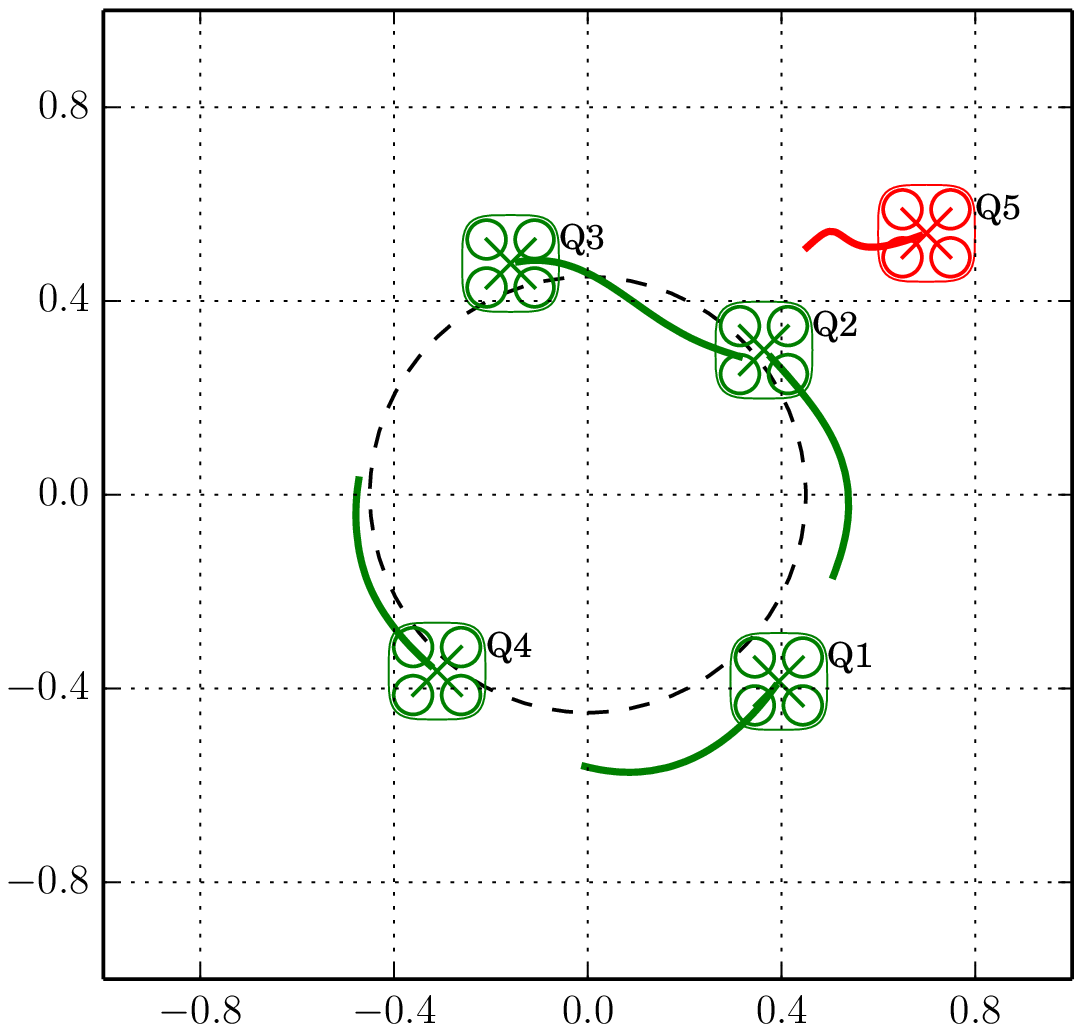}}
  \caption{Agents at 27.5s}
  \label{fig:t4}
\end{subfigure}
\caption{Experimental data of the team of quadrotors plotted in the X-Y plane. The tail of each quadrotor illustrates its trajectory in the past 0.6s.}
\label{fig:exp2}
\end{figure}

Experimental results provided in this section demonstrate that the \textit{Safety Barrier Certificates} can save flight planners the hassle of considering collision avoidance when designing the higher level multi-robot coordination algorithm. This strategy can be easily used in conjunction with other complicated motion planning strategies, e.g., optimal control algorithms, temporal/spatial assignment algorithms, to provide desired safety guarantees. The minimally invasive enforcement of the \textit{Safety Barrier Certificates} ensures that desired controller will not be rectified unless truly necessary.

\section{Conclusions}\label{sec:conclude}
A flight trajectory modification strategy is presented in this paper to ensure collision-free manuevers for teams of differential flatness based quadrotors. To do this, nominal flight trajectories, which are generated with existing control and planning algorithms, are modified in a minimally invasive way using the \textit{Safety Barrier Certificates} to avoid collisions. The proof of existence of feasible controller and virtual vehicle parameterization method to accommodate actuator limits are presented. In the end, experimental implementation of the \textit{Safety Barrier Certificates} on a team of five quadrotors validates the effectiveness of the proposed strategy. To formally prove that the virtural vehicle parameterization method results in feasible flight trajectories subject to actuator limits will be an interesting future direction.


\addtolength{\textheight}{-12cm}   


\bibliographystyle{abbrv}
\bibliography{mybib}

\begin{thebibliography}{10}

\bibitem{crazyflie}
Bitcraze {AB}.
\newblock \url{https://www.bitcraze.io/}.

\bibitem{ames2014CBF}
A.~D. Ames, J.~W. Grizzle, and P.~Tabuada.
\newblock {C}ontrol {B}arrier {F}unction {B}ased {Q}uadratic {P}rograms with
  {A}pplication to {A}daptive {C}ruise {C}ontrol.
\newblock In {\em Decision and Control (CDC), 2014 IEEE 53rd Annual Conference
  on}, pages 6271--6278, Dec 2014.

\bibitem{barr1981superquadrics}
A.~H. Barr.
\newblock Superquadrics and angle-preserving transformations.
\newblock {\em IEEE Computer graphics and Applications}, 1(1):11--23, 1981.

\bibitem{boyd2004convex}
S.~Boyd and L.~Vandenberghe.
\newblock {\em {C}onvex {O}ptimization}.
\newblock Cambridge University Press, 2004.

\bibitem{egerstedt2001control}
M.~Egerstedt, X.~Hu, and A.~Stotsky.
\newblock Control of mobile platforms using a virtual vehicle approach.
\newblock {\em IEEE Transactions on Automatic Control}, 46(11):1777--1782,
  2001.

\bibitem{grammatico2014control}
S.~Grammatico, F.~Blanchini, and A.~Caiti.
\newblock Control-sharing and merging control lyapunov functions.
\newblock {\em IEEE Transactions on Automatic Control}, 59(1):107--119, 2014.

\bibitem{hehn2015real}
M.~Hehn and R.~D’Andrea.
\newblock Real-time trajectory generation for quadrocopters.
\newblock {\em IEEE Transactions on Robotics}, 31(4):877--892, 2015.

\bibitem{HoenigMixedReality2015}
W.~Hoenig, C.~Milanes, L.~Scaria, T.~Phan, M.~Bolas, and N.~Ayanian.
\newblock Mixed reality for robotics.
\newblock In {\em IEEE/RSJ Intl Conf. Intelligent Robots and Systems}, pages
  5382 -- 5387, Hamburg, Germany, Sept 2015.

\bibitem{khalil1996nonlinear}
H.~K. Khalil.
\newblock {\em Nonlinear systems}.
\newblock Prentice hall, third edition, 2002.

\bibitem{kuriki2014consensus}
Y.~Kuriki and T.~Namerikawa.
\newblock Consensus-based cooperative formation control with collision
  avoidance for a multi-uav system.
\newblock In {\em 2014 American Control Conference}, pages 2077--2082. IEEE,
  2014.

\bibitem{landry2015planning}
B.~Landry.
\newblock Planning and control for quadrotor flight through cluttered
  environments.
\newblock Master's thesis, Massachusetts Institute of Technology, 2015.

\bibitem{mellinger2011minimum}
D.~Mellinger and V.~Kumar.
\newblock Minimum snap trajectory generation and control for quadrotors.
\newblock In {\em Robotics and Automation (ICRA), 2011 IEEE International
  Conference on}, pages 2520--2525. IEEE, 2011.

\bibitem{mellinger2012mixed}
D.~Mellinger, A.~Kushleyev, and V.~Kumar.
\newblock Mixed-integer quadratic program trajectory generation for
  heterogeneous quadrotor teams.
\newblock In {\em Robotics and Automation (ICRA), 2012 IEEE International
  Conference on}, pages 477--483. IEEE, 2012.

\bibitem{morris2013sufficient}
B.~Morris, M.~J. Powell, and A.~D. Ames.
\newblock Sufficient conditions for the lipschitz continuity of qp-based
  multi-objective control of humanoid robots.
\newblock In {\em 52nd IEEE Conference on Decision and Control}, pages
  2920--2926. IEEE, 2013.

\bibitem{nguyen2016exponential}
Q.~Nguyen and K.~Sreenath.
\newblock Exponential control barrier functions for enforcing high
  relative-degree safety-critical constraints.
\newblock In {\em 2016 American Control Conference (ACC)}, pages 322--328.
  IEEE, 2016.

\bibitem{Quad:video}
Online.
\newblock Barrier certificates for safe quad swarm.
\newblock \url{https://www.youtube.com/watch?v=rK9oyqccMJw}, 2016.

\bibitem{rockafellar2015convex}
R.~T. Rockafellar.
\newblock {\em Convex analysis}.
\newblock Princeton university press, 2015.

\bibitem{valenti2007mission}
M.~Valenti, D.~Dale, J.~How, and J.~Vian.
\newblock Mission health management for 24/7 persistent surveillance
  operations.
\newblock In {\em AIAA Guidance, Control and Navigation Conference, Myrtle
  Beach, SC}, 2007.

\bibitem{wang2016hetero}
L.~Wang, A.~D. Ames, and M.~Egerstedt.
\newblock {S}afety {B}arrier {C}ertificates for {H}eterogeneous {M}ulti-robot
  {S}ystems.
\newblock In {\em 2016 American Control Conference (ACC)}, pages 5213--5218,
  July 2016.

\bibitem{wang2016multiobj}
L.~Wang, A.~D. Ames, and M.~Egerstedt.
\newblock Multi-objective compositions for collision-free connectivity
  maintenance in teams of mobile robots.
\newblock In {\em Decisions and Control Conference}, 2016, to appear.

\bibitem{wu2016safety}
G.~Wu and K.~Sreenath.
\newblock Safety-critical control of a planar quadrotor.
\newblock In {\em 2016 American Control Conference (ACC)}, pages 2252--2258.
  IEEE, 2016.

\bibitem{Wu2016Quad}
G.~Wu and K.~Sreenath.
\newblock Safety-critical control of a 3d quadrotor with range-limited sensing.
\newblock In {\em ASME Dynamics Systems and Control Conference (DSCC)}, to
  appear, 2016.

\bibitem{Xu2015Robustness}
X.~Xu, P.~Tabuada, J.~W. Grizzle, and A.~D. Ames.
\newblock Robustness of control barrier functions for safety critical control.
\newblock In {\em Analysis and Design of Hybrid Systems, 2015 IFAC Conference
  on}. IEEE, Oct 2015.

\bibitem{zhang2012application}
C.~Zhang and J.~M. Kovacs.
\newblock The application of small unmanned aerial systems for precision
  agriculture: a review.
\newblock {\em Precision agriculture}, 13(6):693--712, 2012.

\bibitem{zhou2014vector}
D.~Zhou and M.~Schwager.
\newblock Vector field following for quadrotors using differential flatness.
\newblock In {\em 2014 IEEE International Conference on Robotics and Automation
  (ICRA)}, pages 6567--6572. IEEE, 2014.

\end{thebibliography}
\end{document}